\documentclass{article}

\usepackage[preprint]{neurips_2026}

\usepackage[utf8]{inputenc} 
\usepackage[T1]{fontenc}    
\usepackage{hyperref}       
\usepackage{url}            
\usepackage{booktabs}       
\usepackage{amsfonts}       
\usepackage{nicefrac}       
\usepackage{microtype}      
\usepackage{xcolor}         
\usepackage{multirow}
\usepackage{subcaption}
\usepackage{graphicx}
\usepackage{hyperref}
\usepackage{colortbl}
\usepackage{amsmath} 
\usepackage{amssymb}
\usepackage{tabularx}
\usepackage{amsthm}
\usepackage{mathtools}
\usepackage{algorithm}
\usepackage{algorithmic}

\newtheorem{theorem}{Theorem}

\title{TAD: Temporal-Aware Trajectory Self-Distillation for Fast and Accurate  Diffusion LLM}

\author{
  \textbf{Haoyang Zhou}\textsuperscript{1} \quad
  \textbf{Li Kong}\textsuperscript{1} \quad
  \textbf{Shijie Ren}\textsuperscript{1} \quad
  \textbf{Xiting Wang}\textsuperscript{1,*} \quad \\
  \textbf{Shuang Liang}\textsuperscript{1} \quad
  \textbf{Guowei Wang}\textsuperscript{2} \quad
  \textbf{Zhenxuan Pan}\textsuperscript{2}
}

\begin{document}

\maketitle
\vspace*{-3em}
\begin{center}
\textsuperscript{1}\text{Gaoling School of Artificial Intelligence, Renmin University of China} \\
\textsuperscript{2}\text{Ant Group} \\
\textsuperscript{*}\text{Corresponding author}
\end{center}

\begin{abstract}
Diffusion large language models (dLLMs) offer a promising paradigm for parallel text generation, but in practice they face an accuracy-parallelism trade-off, where increasing tokens per forward (TPF) often degrades generation quality. Existing acceleration methods often gain speed at the cost of accuracy. To address this limitation, we propose TAD, a Temporal-Aware trajectory self-Distillation framework. During data construction, we condition a teacher model on both the prompt and the ground-truth response to generate decoding trajectories, recording the intermediate masked states throughout the process. Based on how many decoding steps remain before each masked token is revealed, we partition masked positions into near and distant subsets. For near tokens, we train the student with a hard cross-entropy loss using the teacher trajectory tokens as labels, encouraging confident predictions for tokens that are about to be decoded. For distant tokens, we apply a soft KL divergence loss between the teacher and student token distributions, providing softer supervision and preserving future planning knowledge. This temporal-aware partition naturally gives rise to two deployment configurations: a Quality model that prioritizes accuracy and a Speed model that favors more aggressive acceleration. Experiments show that TAD consistently improves the accuracy-parallelism trade-off. On LLaDA, it raises average accuracy from 46.2\% to 51.6\% with the Quality model and average AUP from 46.2 to 257.1 with the Speed model. Our code is available at: \href{https://github.com/BHmingyang/TAD}{https://github.com/BHmingyang/TAD}.

\end{abstract}

\section{Introduction} 
Diffusion large language models (dLLMs)~\cite{austin2021structured,lou2024discrete,sahoo2024simple,shi2024simplified,nie2025large} have recently emerged as a promising alternative to Autoregressive (AR) language models. Unlike AR models that generate tokens strictly from left to right, dLLMs inherently support bidirectional attention and parallel generation of multiple tokens. Despite this theoretical potential, achieving high parallelism in practice remains a challenge~\cite{kang2025parallelbench}. In each forward pass, dLLMs predict masked tokens independently, ignoring the sequential dependencies of natural language~\cite{wu2025fast}. This mismatch leads to a substantial decline in generation quality when decoding multiple tokens simultaneously. To preserve robust performance, existing open-source dLLMs~\cite{nie2025large,ye2025dream} often apply conservative decoding schedules requiring hundreds of denoising steps, which exposes a critical gap between potential parallelism and realized throughput.

\begin{figure}[t]
    \centering
    \includegraphics[width=1\textwidth]{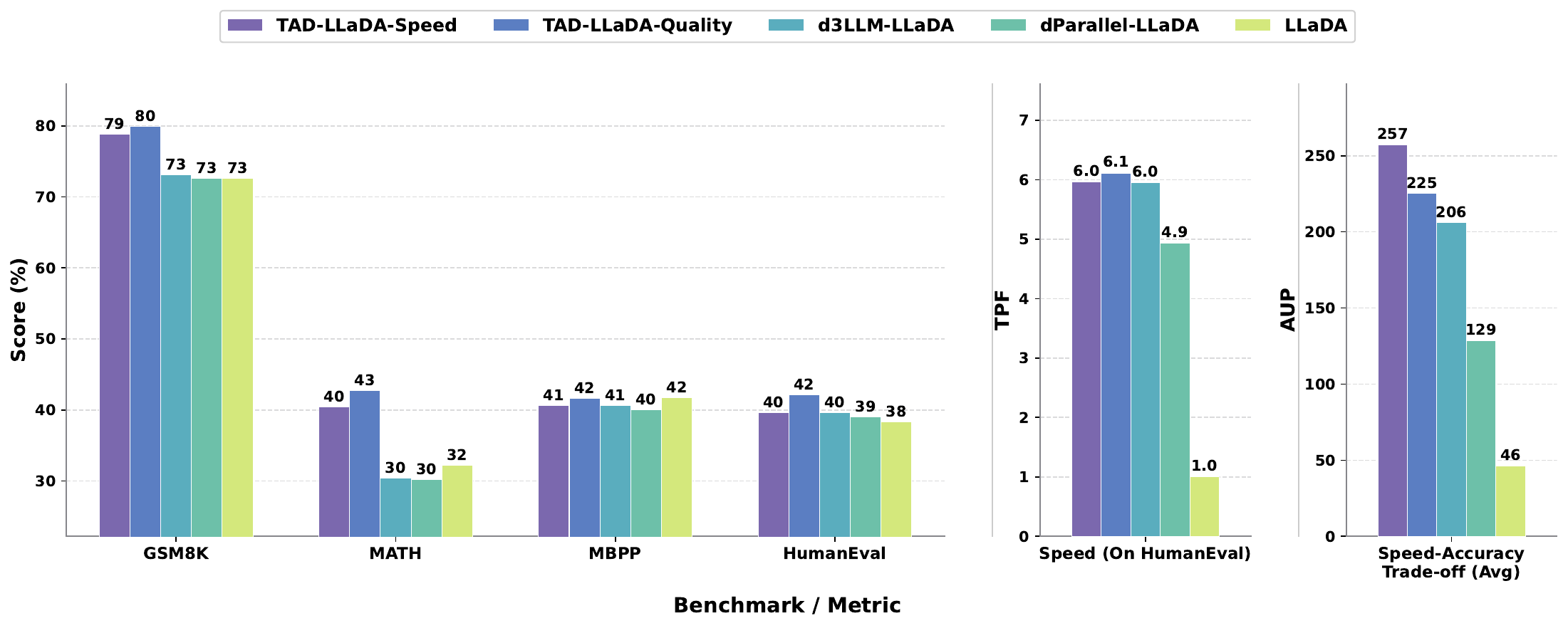}
    \caption{\textbf{TAD} achieves competitive accuracy on math and code benchmarks, significantly improving the Accuracy under Parallelism (AUP)~\cite{qian2026d3llm} compared to baselines like d3LLM~\cite{qian2026d3llm} and dParallel~\cite{chen2025dparallel}.}
    \label{fig:acc_aup_dual}
\end{figure}

To narrow this gap, recent studies have explored both training-free~\cite{hong2025wide,wu2025fast,xu2025lopa,mohamed2025fast} and training-based methods~\cite{chen2025dparallel,zhang2026t3d,qian2026d3llm,wang2025diffusion} for few-step parallel decoding. Training-free methods offer plug-and-play acceleration, but their effectiveness is bounded by the capacity of the base model and often suffer quality degradation under aggressive decoding~\cite{lin2026efficient}. Training-based methods can push parallelism further by finetuning the model for fast decoding, typically through trajectory distillation, as in dParallel~\cite{chen2025dparallel} and d3LLM~\cite{qian2026d3llm}. However, these gains often come at the cost of generation quality due to two key limitations. First, standard trajectory collection conditions solely on the prompt, meaning the quality of the generated trajectories is strictly upper-bounded by the base model's inherent reasoning capacity. This self-generation bottleneck prevents the model from learning knowledge beyond its current limits. Second, existing methods do not fully use the temporal information of decoding: they either apply the cross-entropy loss to all masked tokens~\cite{qian2026d3llm}, which can force early decisions on inherently uncertain positions, or ignore supervision on distant tokens~\cite{chen2025dparallel}, which removes useful signals about future content.

Motivated by these limitations, we propose TAD, a Temporal-Aware trajectory self-Distillation framework for dLLMs. To improve trajectory quality, we adapt the privileged-information strategy~\cite{zhao2026self,hubotter2026reinforcement,shenfeld2026self} to dLLMs in which the teacher and student share the same backbone, while the teacher receives the ground-truth response as additional privileged information. This design enables the teacher to generate high-quality decoding trajectories, from which we extract the intermediate masked states for subsequent training. To better utilize the temporal structure of decoding, we partition masked positions by how many decoding steps remain before they are revealed. We apply cross-entropy to near tokens using teacher trajectory tokens as labels, encouraging confident predictions for tokens that will be decoded soon. For distant tokens, we use a KL divergence objective to align the student with the teacher distribution, providing a smoothed learning signal that avoids overfitting. This temporal-aware design yields two flexible deployment configurations: a Quality model that prioritizes accuracy and a Speed model that favors more aggressive acceleration.

We evaluate TAD on math and code generation tasks using LLaDA~\cite{nie2025large} and Dream~\cite{ye2025dream}. Comprehensive experiments show that our approach effectively improves the accuracy-parallelism trade-off. For instance, when applied to LLaDA~(Fig.~\ref{fig:acc_aup_dual}), TAD-LLaDA-Speed achieves a $6\times$ speedup while improving the accuracy from 38\% to 42\% on HumanEval~\cite{chen2021evaluating}. Furthermore, across all evaluated tasks, TAD achieves the best average accuracy under parallelism (AUP) scores.

Our main contributions are as follows:
\begin{itemize}
\item We adapt the privileged-information strategy to dLLMs by leveraging ground-truth as privileged information to generate high-quality trajectories for distillation.
\item We propose a temporal-aware trajectory self-distillation framework to partition masked positions based on their decoding steps, applying cross-entropy to near tokens to maximize throughput and soft KL divergence to distant tokens to preserve sequence dependencies.
\item We conduct extensive experiments and show that TAD improves the accuracy-parallelism trade-off, surpassing the average accuracy of the base models while achieving the highest Accuracy Under Parallelism (AUP) scores.
\end{itemize}

\section{Preliminaries}
\textbf{Masked Diffusion Language Models (MDLMs).} 
MDLMs~\cite{austin2021structured,shi2024simplified,nie2025large} formulate text generation as a probabilistic process comprising forward corruption and reverse denoising. Given a clean token sequence $\mathbf{x_0}=(x_0^1,\ldots,x_0^L)$, the forward process constructs intermediate states by replacing tokens with the special token \texttt{[MASK]}. For a corruption level $t \sim \mathcal{U}(0,1)$, the sequence $\mathbf{x}_t$ is partially masked with each position being masked with probability $t$. The conditional distribution for $\mathbf{x}_t$ is:

\begin{equation}
q(\mathbf{x}_t \mid \mathbf{x}_0) = \prod_{i=1}^{L} q(x_t^i \mid x_0^i), \quad \text{where } q(x_t^i \mid x_0^i) = \begin{cases} 1-t, & \text{if } x_t^i = x_0^i \\ t, & \text{if } x_t^i = \texttt{[MASK]} \end{cases}.
\end{equation}

The reverse denoising process reconstructs the clean sequence by employing a parameterized model $p_\theta$ that estimates the conditional distribution of the masked tokens given $\mathbf{x}_t$. This enables the parallel prediction of all masked positions:

\begin{equation}
\label{eq:factorized}
p_\theta(\mathbf{x}_0 \mid \mathbf{x}_t) = \prod_{i: x_t^i = \texttt{[MASK]}} p_\theta(x_0^i \mid \mathbf{x}_t).
\end{equation}
The model is optimized by minimizing the negative log-likelihood over these masked tokens~\cite{nie2025large}:
\begin{equation}
\label{eq:objectives}
\mathcal{L}(\theta) = -\mathbb{E}_{t \sim \mathcal{U}(0,1),\; \mathbf{x}_0 \sim p_{\text{data}},\; \mathbf{x}_t \sim q(\cdot \mid \mathbf{x}_0, t)} \left[ \frac{1}{t} \sum_i \mathbf{1}[x_t^i = \texttt{[MASK]}] \log p_\theta(x_0^i \mid \mathbf{x}_t) \right].
\end{equation}
Here, $\mathbf{1}[\cdot]$ denotes the indicator function, which ensures the loss is computed only for masked tokens.

\textbf{The Factorization Gap in Parallel Decoding.}
Parallel decoding relies on the independence assumption in Equation~\ref{eq:factorized}. However, the true posterior distribution rarely factorizes. Let $\mathcal{M}(\mathbf{x}_t)$ denote the set of $K$ masked positions and $\mathbf{x}_0^\mathcal{M}$ denote the corresponding clean tokens. Using the chain rule of probability for an arbitrary decoding order $(m_1, \ldots, m_K)$, the true joint distribution is:
\begin{equation}
\label{eq:joint_distribution}
p^{\ast}(\mathbf{x}_0^\mathcal{M} \mid \mathbf{x}_t) = \prod_{k=1}^{K} p^{\ast}(x_0^{m_k} \mid \mathbf{x}_t, x_0^{m_1}, \ldots, x_0^{m_{k-1}}).
\end{equation}
This equation captures the sequential dependencies among tokens. Since the training objective in Equation~\ref{eq:objectives} optimizes each position individually, the resulting model ignores these dependencies. This mismatch can be formulated as the factorization gap\cite{kang2025parallelbench}:
\begin{equation}
\label{eq:factorization_gap}
\Delta_{\text{gap}}(\mathbf{x}_t) \;=\; D_{\text{KL}}\!\left(p^{\ast}(\mathbf{x}_0^\mathcal{M} \mid \mathbf{x}_t) \;\Big\|\; \prod_{i \in \mathcal{M}} p^{\ast}(x_0^i \mid \mathbf{x}_t)\right).
\end{equation}
While the gap is manageable when predicting only a few tokens, it expands significantly during aggressive parallel decoding and causes a severe drop in generation quality~\cite{wu2025fast,kang2025parallelbench}.

\section{Method}
In this section, we present the Temporal Aware Self-Distillation (TAD) framework. We begin in Section \ref{sec:formulation} by introducing our motivation. Section \ref{sec:guided_trajectory} details a trajectory collection mechanism that utilizes ground-truth as privileged information to sample high-quality intermediate states. Finally, Section \ref{sec:self_distillation} introduces our temporal-aware distillation strategy in detail.

\begin{figure*}[t]
    \centering
    \includegraphics[width=0.95\textwidth]{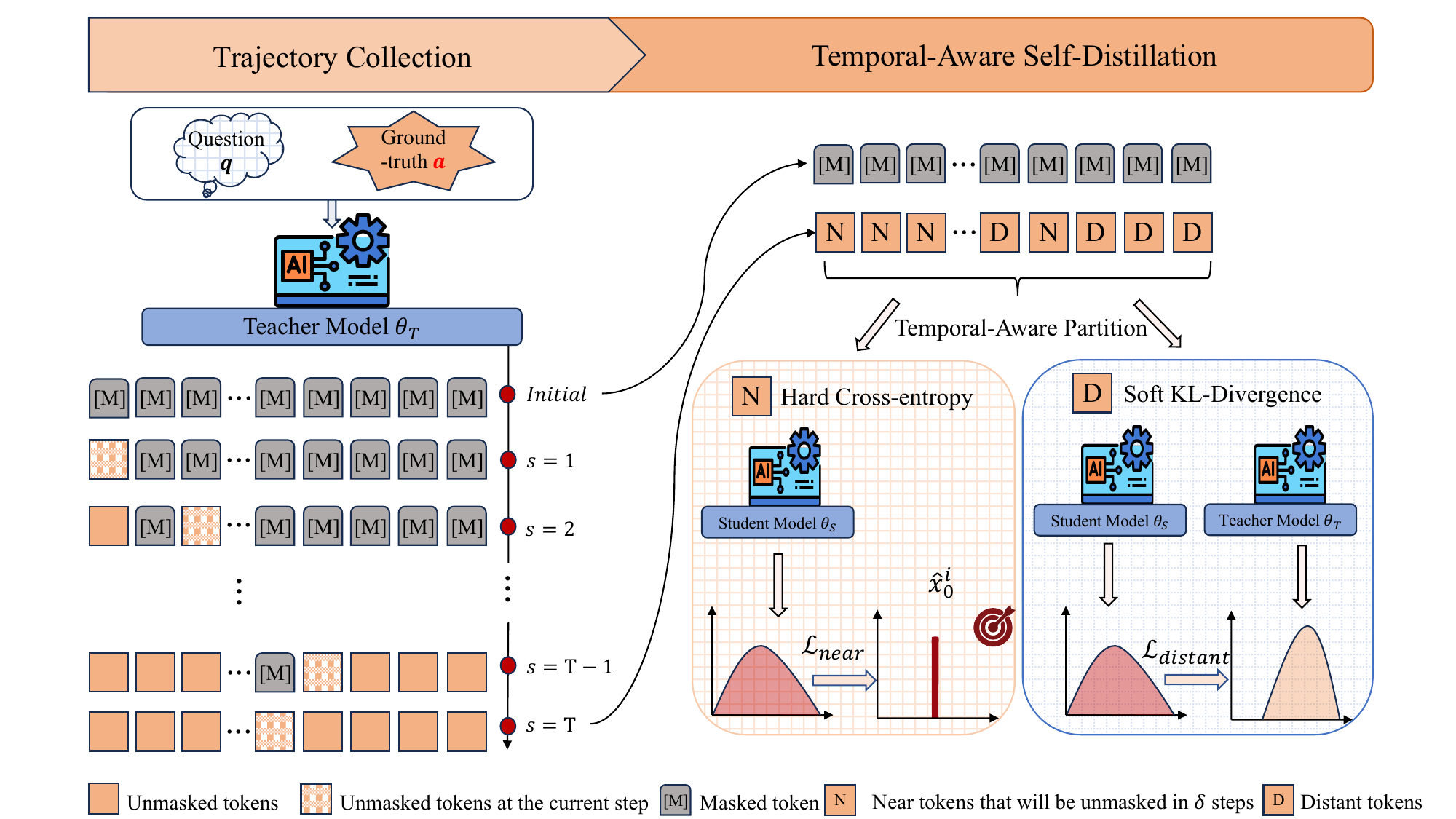}
    \caption{Overview of the TAD framework. The pipeline consists of two primary phases. \textbf{Left} (Trajectory Collection via Privileged Information): A teacher model ($\theta_T$), conditioned on the prompt $q$ and ground-truth sequence $a$, decodes exactly one token per step. \textbf{Right} (Temporal-Aware Self-Distillation): The masked positions are spatially partitioned into a near subset (N) and a distant subset (D) based on a predefined window $\delta$. The student model ($\theta_S$) is optimized using a hard cross-entropy objective on the near tokens to maximize parallel decoding throughput, and a soft KL divergence objective on the distant tokens to align with the teacher distribution.}
    \label{fig:tad_framework}
\end{figure*}

\subsection{Motivation}
\label{sec:formulation}
The factorization gap in Eq.~\ref{eq:factorization_gap} represents a fundamental limitation of parallel decoding in dLLMs. Even a perfect marginal learner still incurs $\Delta_{\text{gap}}(\mathbf{x}_t) > 0$ whenever the masked set contains contextually dependent tokens~\cite{kang2025parallelbench}. To reduce this gap, a factorized model should internalize the dependency structure into its marginals by learning from the joint distribution (Eq.~\ref{eq:joint_distribution}). We therefore consider a teacher model that performs token-by-token (TBT) decoding through a strict remasking strategy, which constructs a sequential joint distribution along a specific decoding path:
\begin{equation}
\label{eq:tbt_joint}
p_{\theta_{T}}^{TBT}(\mathbf{x}_0^{\mathcal{M}} \mid \mathbf{x}_t, q) = \prod_{k=1}^K p_{\theta_{T}}(x_0^{m_k} \mid \mathbf{x}_t, x_0^{m_1}, \ldots, x_0^{m_{k-1}}, q),
\end{equation}
where $m_k$ denotes the position index of the token unmasked at the $k$-th decoding step. Our objective is to align the student distribution $p_{\theta_{S}}(\mathbf{x}_0^{\mathcal{M}} \mid \mathbf{x}_t, q)$ with this sequential distribution by minimizing their expected divergence over a set of trajectories $\tau$:
\begin{equation}
\label{eq:theoretical_obj}
\min_{\theta_S} \mathbb{E}_{\mathbf{x}_t \sim \tau} \left[ D_{\text{KL}} \left( p_{\theta_{T}}^{TBT}(\mathbf{x}_0^{\mathcal{M}} \mid \mathbf{x}_t, q) \parallel p_{\theta_{S}}(\mathbf{x}_0^{\mathcal{M}} \mid \mathbf{x}_t, q) \right) \right].
\end{equation}

As formally proved in Appendix \ref{sec:proof}, this joint divergence mathematically simplifies to the expected sum of token-wise cross-entropies along the decoding trajectory:
\begin{equation}
\label{eq:simplified_obj}
\min_{\theta_S} \mathbb{E}_{\mathbf{x}_t \sim \tau} \sum_{k=1}^K \mathbb{E}_{x_{<k} \sim p_{\theta_T}^{TBT}} \left[ -\sum_{x_0^{m_k}} p_{\theta_T}(x_0^{m_k} \mid \mathbf{x}_t, x_{<k}, q) \log p_{\theta_S}(x_0^{m_k} \mid \mathbf{x}_t, q) \right].
\end{equation}
However, translating this theoretical objective into an effective training framework still faces two challenges.  First, standard TBT decoding conditioned solely on the prompt $q$ is strictly upper-bounded by the base model's reasoning capabilities. Constrained by this boundary, the model lacks the capacity to reliably explore correct paths in complex tasks. Second, the conditional targets in Equation~\ref{eq:simplified_obj} are often highly peaked (Fig.~\ref{fig:teacher_conf}), as the teacher relies on the preceding context $x_{<k}$. Imposing these deterministic labels across all masked positions forces the student model to guess distant tokens before establishing the necessary local context, which may cause overfitting. The TAD framework addresses these challenges through privileged information guidance and temporal partitioning (Fig.~\ref{fig:tad_framework}).

\begin{figure}[t]
    \centering
    \begin{subfigure}[b]{0.49\textwidth}
        \centering
        \includegraphics[height=4.5cm,width=\textwidth]{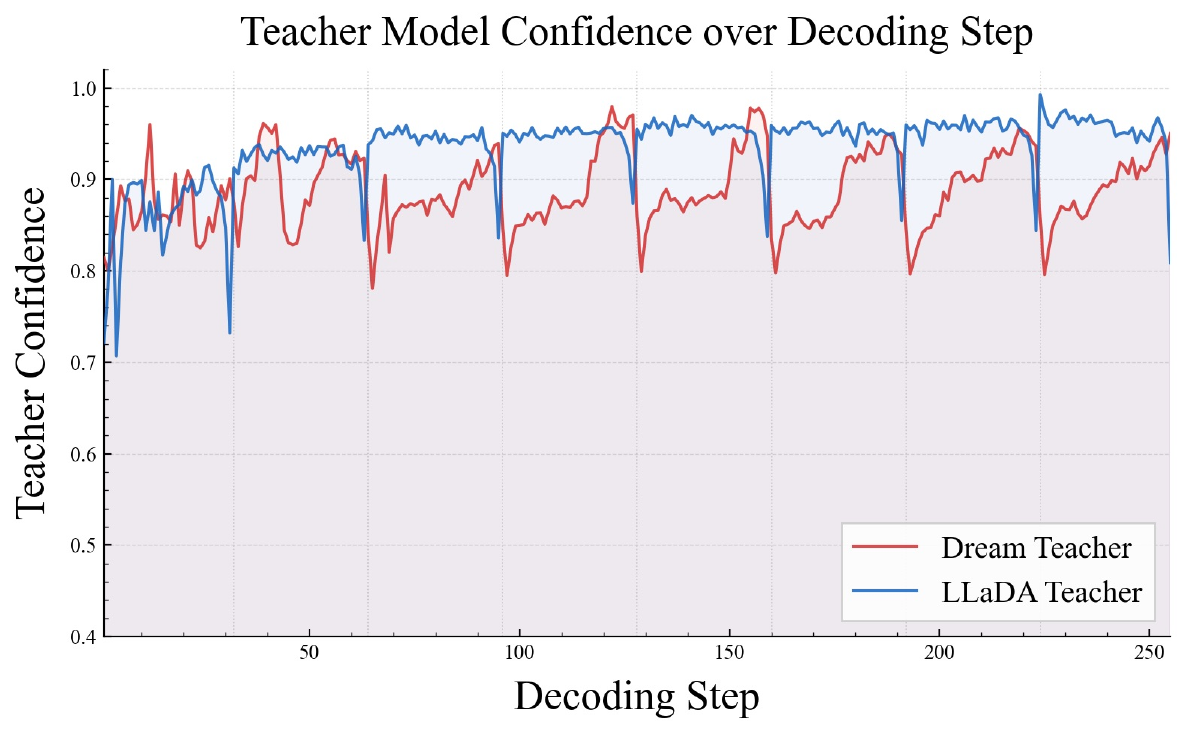}
        \caption{Teacher's per-step confidence on collected data.}
        \label{fig:teacher_conf}
    \end{subfigure}
    \hfill
    \begin{subfigure}[b]{0.49\textwidth}
        \centering
        \includegraphics[height=4.5cm,width=\textwidth]{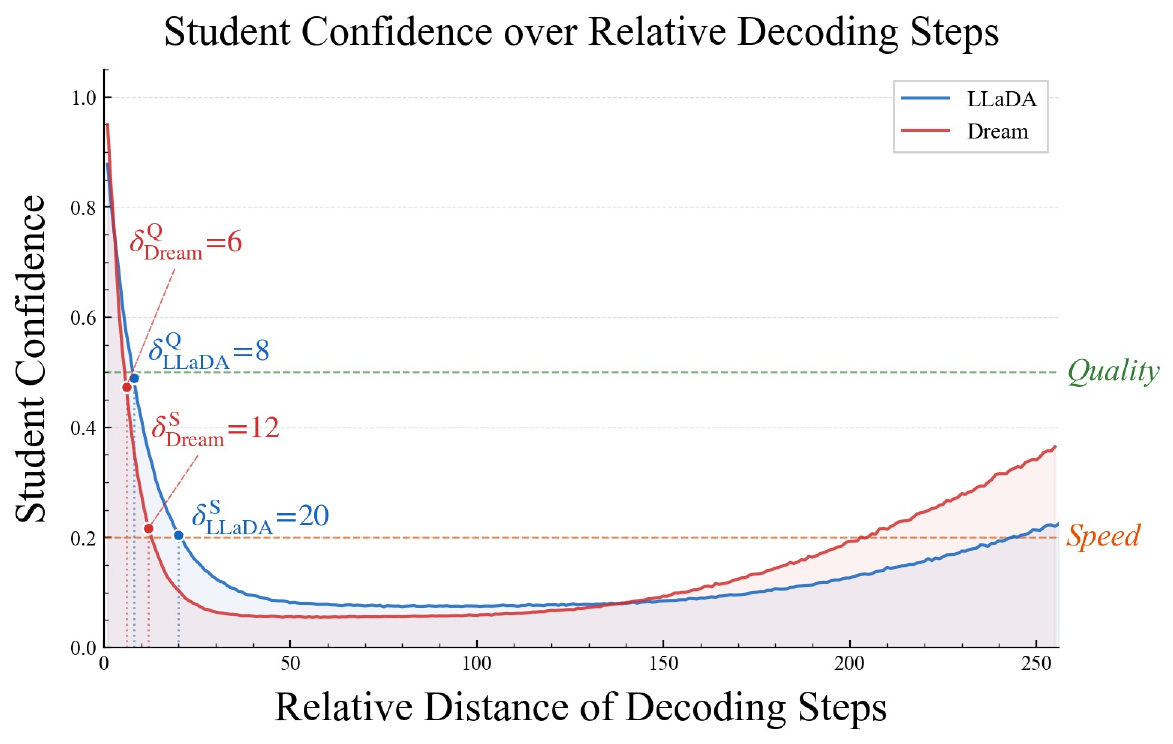}
        \caption{Student confidence on collected data.}
        \label{fig:step_diff_conf}
    \end{subfigure}
    \caption{
        \textbf{(a)} The average per-step confidence reaches 0.93 for the LLaDA teacher and 0.88 for the Dream, confirming that the conditional targets in
        Equation~\ref{eq:simplified_obj} are highly peaked.
        \textbf{(b)} The selection of the spatial partition window $\delta$ is guided by the student's confidence decay on the teacher's predicted tokens. We derive two configurations based on this decay curve: $\delta_{quality}$ is set at the relative step where the student's confidence drops to 0.5, and $\delta_{speed}$ is set where the confidence degrades to 0.2.
    }
    \label{fig:confidence_analysis}
\end{figure}

\subsection{Trajectory Collection via Privileged Information}
\label{sec:guided_trajectory}
Inspired by recent self-distillation work~\cite{zhao2026self}, we introduce a data collection mechanism guided by privileged information to address the trajectory quality challenge identified in Section~\ref{sec:formulation}. Instead of relying solely on the prompt $q$, we provide the teacher model with the ground-truth sequence $a$ as privileged information. To obtain the intermediate states of these trajectories for student optimization, we constrain the teacher to perform a discrete token-by-token rollout over $T$ decoding steps. Let $\mathbf{x}_s$ denote the intermediate masked state at decoding step $s$, where $s \in \{1, 2, \ldots, T\}$. The initial step consists entirely of the prompt and mask tokens. At each step $s$, the teacher computes the conditional probability $p_{\theta_T}(x_0^i \mid \mathbf{x}_s, q, a)$ for all currently masked positions $i \in \mathcal{M}(\mathbf{x}_s)$. To capture the sequential dependencies, we employ a strict remasking strategy that unmasks one token per step. We select the target position $m_s$ by identifying the token with the highest confidence:
\begin{equation}
m_s = \operatorname*{arg\,max}_{i \in \mathcal{M}(\mathbf{x}_s)} \max_{v} p_{\theta_T}(x_0^i = v \mid \mathbf{x}_s, q, a),
\end{equation}
where $v$ represents a token from the vocabulary. After selecting $m_s$ and determining its corresponding token $\hat{x}^{m_s} = \arg\max_v p_{\theta_T}(x_0^{m_s} = v \mid \mathbf{x}_s, q, a)$, we transition to the subsequent state $\mathbf{x}_{s+1}$ by replacing the mask at $m_s$ with $\hat{x}^{m_s}$.

We record the intermediate states and the selected tokens produced throughout this sequential procedure to construct the training trajectory $\tau_{\text{priv}} = \{ (\mathbf{x}_s, m_s, \hat{x}^{m_s}) \}_{s=1}^T$. By capturing the step-by-step token decisions guided by the ground truth, this process directly supplies the reliable conditional targets for the subsequent temporal-aware distillation phase.

\subsection{Temporal-Aware Self-Distillation}
\label{sec:self_distillation}
With the trajectories $\tau_{\text{priv}}$ collected, we implement a temporal-aware self-distillation strategy to address the optimization instability. This approach approximates the theoretical objective of conditional probability while preventing the student model from memorizing uncertain predictions. At each trajectory step $s \in \{1, 2, \ldots, T\}$, we partition the currently masked positions $\mathcal{M}(\mathbf{x}_s)$ into a near subset $\mathcal{M}_{near}$ and a distant subset $\mathcal{M}_{distant}$ by how many decoding steps remain before they are revealed. This partition relies on a predefined look-ahead window $\delta$. We define a binary mask indicator $b_s^{(i)} \in \{0, 1\}$ for each token position $i$, where $1$ represents a masked state. For any step exceeding the maximum decoding length ($s+\delta > T$), we naturally define $b_{s+\delta}^{(i)} = 0$ because the entire sequence is fully unmasked by the final step. This predefined horizon strictly divides the sequence into two mutually exclusive target subsets:
\begin{align}
    \mathcal{M}_{near} &= \{i \mid b_s^{(i)}=1 \land b_{s+\delta}^{(i)}=0\}, \\
    \mathcal{M}_{distant} &= \{i \mid b_s^{(i)}=1 \land b_{s+\delta}^{(i)}=1\}.
\end{align}
The subset $\mathcal{M}_{near}$ contains tokens about to be decoded within the temporal window, whereas $\mathcal{M}_{distant}$ contains tokens remaining masked within $\delta$ steps. 

For the near subset, we approximate the highly peaked conditional probabilities of the teacher using hard labels. Applying a cross-entropy loss with the exact tokens selected by the teacher anchors the student model to a valid decoding path. This objective drives the student to achieve high certainty on near tokens, which maximizes parallel decoding throughput. The near loss is formulated as:
\begin{equation}
\mathcal{L}_{near} = - \mathbb{E}_{\mathbf{x}_s \sim \tau_{\text{priv}}} \left[ \sum_{i \in \mathcal{M}_{near}} \log p_{\theta_{S}}(\hat{x}_{0}^i \mid \mathbf{x}_s, q) \right],
\end{equation}
where $\hat{x}_{0}^i$ denotes the reference token from the collected teacher trajectory. 

For the distant subset, we relax the strict trajectory-dependent target to prevent the student from overfitting. We utilize the single-step factorized output of the teacher $p_{\theta_{T}}(x_0^j \mid \mathbf{x}_s, q, a)$ at each position $j$ as a soft proxy. We enforce alignment by calculating the Kullback-Leibler divergence between this proxy and the prediction of the student. This soft constraint provides a smoothed learning signal without forcing deterministic predictions on uncertain tokens. The distant loss is defined as:
\begin{equation}
\mathcal{L}_{distant} = \mathbb{E}_{\mathbf{x}_s \sim \tau_{\text{priv}}} \left[ \sum_{j \in \mathcal{M}_{distant}} D_{\text{KL}} \left( p_{\theta_{T}}(x_0^j \mid \mathbf{x}_s, q, a) \parallel p_{\theta_{S}}(x_0^j \mid \mathbf{x}_s, q) \right) \right].
\end{equation}

The final objective of the TAD framework is a weighted combination of these two components:
\begin{equation}
\mathcal{L}_{TAD} = \mathcal{L}_{near} + \lambda \mathcal{L}_{distant},
\end{equation}
where $\lambda$ constitutes a balancing hyperparameter. This formulation adapts the optimization objective to the temporal dynamics of parallel generation, ensuring high parallelism through near-term certainty and robust generation quality through long-term dependency preservation.

\textbf{The Choice of the Temporal Partition Window $\delta$.}
The temporal partition window $\delta$ dictates the boundary between near subset and distant subset. To determine an appropriate value for $\delta$, we analyze how the student model's predictive confidence on the reference tokens from the collected trajectories decays relative to the decoding steps, as illustrated in Figure~\ref{fig:step_diff_conf}. The student exhibits high confidence for tokens that are immediately adjacent in the decoding steps, but this certainty degrades as the temporal distance increases. Guided by this empirical decay curve, we derive two operational configurations for the TAD framework. We establish a Quality Mode by setting $\delta_{quality}$ to the relative step where the inherent confidence of the student drops to 0.5, ensuring that hard labels are only applied where the model retains moderate natural certainty. Alternatively, we establish a Speed Mode by extending the window to $\delta_{speed}$, corresponding to the point where confidence degrades to 0.2, aggressively expanding the near subset to maximize parallel throughput at a slight cost to accuracy.

\section{Experiment}
\label{sec: experiment}
\subsection{Experimental Details}
\label{subsec: experiment details}
\textbf{Training Dataset.} As a trajectory-distillation method, we use prompts and ground-truth answers from public datasets, and let the model generate its own responses to construct the training data. For LLaDA-8B-Instruct~\cite{nie2025large}, we sample prompts and ground-truth answers from the training splits of GSM8K~\cite{cobbe2021gsm8k}, PRM12K~\cite{lightman2023let} and a subset of KodCode~\cite{xu2025kodcode}. We generate target trajectories with a sequence length of 256 and a block length of 32. We adopt a low-confidence remasking strategy, in which the model generates only one token at each step. We record the intermediate state at every decoding step. We then filter out a number of incorrect trajectories, most of which are caused by formatting errors, and obtain approximately 26k training samples. For Dream-7B-Instruct~\cite{ye2025dream}, we use the same trajectory collection strategy and obtain approximately 22k training samples.

\textbf{Training Details.} In all experiments, we fine-tune the models using LoRA~\cite{hu2022lora} with a rank of $r=128$ and a scaling factor of $\alpha=128$. We configure the spatial partition window $\delta$ according to the selected operational mode and the base architecture. Specifically, for the Quality model, we set $\delta=8$ for LLaDA and $\delta=6$ for Dream. For the Speed model, we set $\delta=20$ and $\delta=12$ for LLaDA and Dream, respectively. All training processes are executed on 8 NVIDIA DGX H200 GPUs. Additional configuration details are available in the Appendix \ref{sec:implementation_details}.

\textbf{Evaluation Details.} We evaluate our model on four widely used benchmarks: GSM8K~\cite{cobbe2021gsm8k}, MATH~\cite{hendrycksmath2021}, HumanEval~\cite{chen2021evaluating}, and MBPP~\cite{austin2021program}. Following prior work~\cite{qian2026d3llm}, we apply a 4-shot setting for MATH and a 3-shot setting for LLaDA on MBPP, while using a 0-shot setting for the remaining evaluations. For inference, we set the generation length to 256 and the block length to 32, adopt entropy-based dynamic decoding with a threshold of 0.5, and employ a multi-block decoding strategy consistent with D2F~\cite{wang2025diffusion} and d3LLM~\cite{qian2026d3llm}. We report accuracy, tokens per forward (TPF), and the Accuracy Under Parallelism (AUP) score~\cite{qian2026d3llm} to provide a comprehensive measurement of model performance and acceleration.

\textbf{Baselines.} Our baselines fall into three categories: (1)The original dLLMs, LLaDA-Insturct~\cite{nie2025large} and Dream-Instruct~\cite{ye2025dream}; (2)Training-free methods: Fast-dLLM~\cite{wu2025fast}; (3)Training-based methods, including Fast-dLLM-v2~\cite{wu2025fastv2}, D2F~\cite{wang2025diffusion}, dParallel~\cite{chen2025dparallel}and d3LLM~\cite{qian2026d3llm}. For a fair and consistent comparison, we directly report the baseline results disclosed in the d3LLM paper~\cite{qian2026d3llm}.

\begin{table*}[t]
\centering
\caption{Comparison of \textit{TAD-LLaDA} with other LLaDA-based models. The best results among acceleration methods are highlighted in \textbf{bold} and the second best are \underline{underlined}.}
\label{tab:main_results}
\resizebox{\textwidth}{!}{%
\setlength{\tabcolsep}{3pt} 
\begin{tabular}{l|ccc|ccc|ccc|ccc|ccc}
\toprule
 & \multicolumn{3}{c|}{\textbf{GSM8K-CoT}} & \multicolumn{3}{c|}{\textbf{MATH}} & \multicolumn{3}{c|}{\textbf{MBPP}} & \multicolumn{3}{c|}{\textbf{HumanEval}} & \multicolumn{3}{c}{\textbf{Average}} \\
\cmidrule(lr){2-4} \cmidrule(lr){5-7} \cmidrule(lr){8-10} \cmidrule(lr){11-13} \cmidrule(lr){14-16}
\textbf{Method} & Acc$\uparrow$ & TPF$\uparrow$ & AUP$\uparrow$ & Acc$\uparrow$ & TPF$\uparrow$ & AUP$\uparrow$ & Acc$\uparrow$ & TPF$\uparrow$ & AUP$\uparrow$ & Acc$\uparrow$ & TPF$\uparrow$ & AUP$\uparrow$ & Acc$\uparrow$ & TPF$\uparrow$ & AUP$\uparrow$ \\
\midrule
LLaDA & 72.6 & 1.00 & 72.6 & 32.2 & 1.00 & 32.2 & 41.7 & 1.00 & 41.7 & 38.3 & 1.00 & 38.3 & 46.2 & 1.00 & 46.2 \\
\midrule
Fast-dLLM & 74.7 & 2.77 & 182.6 & 30.8 & 1.97 & 45.7 & 38.6 & 2.13 & 56.6 & 37.8 & 2.56 & 54.0 & 45.5 & 2.36 & 84.7 \\
D2F & 73.2 & 2.88 & 189.2 & 28.7 & 2.38 & 47.2 & 38.0 & 1.94 & 53.0 & 36.6 & 2.69 & 62.0 & 44.1 & 2.47 & 87.9 \\
dParallel & 72.6 & 5.14 & 308.7 & 30.2 & 3.17 & 61.3 & 40.0 & 2.35 & 60.5 & 39.0 & 4.93 & 83.7 & 45.5 & 3.90 & 128.6 \\
d3LLM & 73.1 & \textbf{9.11} & \underline{539.2} & 30.4 & \textbf{5.74} & 100.1 & \underline{40.6} & \textbf{4.21} & \underline{88.4} & \underline{39.6} & \underline{5.95} & 96.6 & 45.9 & \textbf{6.25} & 206.1 \\
\rowcolor{gray!15}
\textbf{TAD-Q} & \textbf{79.9} & 6.32 & 503.8 & \textbf{42.7} & 4.49 & \underline{190.4} & \textbf{41.6} & 3.38 & \textbf{89.3} & \textbf{42.1} & \textbf{6.11} & \textbf{117.4} & \textbf{51.6} & 5.08 & \underline{225.2} \\

\rowcolor{gray!15}
\textbf{TAD-S} & \underline{78.8} & \underline{8.47} & \textbf{649.0} & \underline{40.4} & \underline{5.17} & \textbf{191.0} & \underline{40.6} & \underline{3.42} & 80.3 & \underline{39.6} & \underline{5.96} & \underline{108.0} & \underline{49.9} & \underline{5.76} & \textbf{257.1} \\

\bottomrule
\end{tabular}%
}
\end{table*}

\begin{table*}[t]
\centering
\caption{Comparison of \textit{TAD-Dream} with other Dream-based models. The best results among acceleration methods are highlighted in \textbf{bold} and the second best are \underline{underlined}.}
\label{tab:TAD-dream-wide}
\resizebox{\textwidth}{!}{%
\setlength{\tabcolsep}{3pt} 
\begin{tabular}{l|ccc|ccc|ccc|ccc|ccc}
\toprule
 & \multicolumn{3}{c|}{\textbf{GSM8K-CoT}} & \multicolumn{3}{c|}{\textbf{MATH}} & \multicolumn{3}{c|}{\textbf{MBPP}} & \multicolumn{3}{c|}{\textbf{HumanEval}} & \multicolumn{3}{c}{\textbf{Average}} \\
\cmidrule(lr){2-4} \cmidrule(lr){5-7} \cmidrule(lr){8-10} \cmidrule(lr){11-13} \cmidrule(lr){14-16}
\textbf{Method} & Acc$\uparrow$ & TPF$\uparrow$ & AUP$\uparrow$ & Acc$\uparrow$ & TPF$\uparrow$ & AUP$\uparrow$ & Acc$\uparrow$ & TPF$\uparrow$ & AUP$\uparrow$ & Acc$\uparrow$ & TPF$\uparrow$ & AUP$\uparrow$ & Acc$\uparrow$ & TPF$\uparrow$ & AUP$\uparrow$ \\
\midrule
Dream & 83.9 & 1.00 & 83.9 & 39.6 & 1.00 & 39.6 & 57.2 & 1.00 & 57.2 & 55.2 & 1.00 & 55.2 & 59.0 & 1.00 & 59.0 \\
\midrule
Fast-dLLM & 79.0 & 1.44 & 116.5 & 38.3 & 1.78 & 55.2 & 53.2 & 1.20 & 63.6 & 54.3 & 1.33 & 63.5 & 56.2 & 1.44 & 74.7 \\
Fast-dLLM-v2 & 77.5 & 2.21 & 156.0 & \textbf{48.7} & 2.61 & 126.7 & 50.1 & 2.04 & 81.9 & 61.7 & 2.58 & 128.9 & 59.5 & 2.36 & 123.4 \\
dParallel & \textbf{82.1} & 3.02 & 245.7 & 38.7 & 2.94 & 77.9 & 55.4 & 2.24 & 108.0 & 54.3 & 2.57 & 98.8 & 57.6 & 2.69 & 132.6 \\
d3LLM & \underline{81.4} & \textbf{4.94} & \textbf{391.3} & 38.2 & 3.92 & 97.5 & \underline{55.6} & \textbf{2.96} & \textbf{141.4} & 57.1 & \textbf{3.20} & 129.5 & 58.1 & \underline{3.76} & \underline{189.9} \\
\rowcolor{gray!15}
\textbf{TAD-Q} & \underline{81.4} & 3.26 & 255.6 & \underline{42.8} & \underline{3.96} & \underline{144.2} & \textbf{56.8} & 2.32 & \underline{122.8} & \textbf{64.1} & 2.35 & \underline{134.3} & \textbf{61.3} & 2.97 & 164.2 \\
\rowcolor{gray!15}
\textbf{TAD-S} & 81.0 & \underline{4.59} & \underline{359.5} & 41.3 & \textbf{5.30} & \textbf{161.3} & 53.8 & \underline{2.38} & 112.9 & \underline{62.2} & \underline{2.90} & \textbf{147.5} & \underline{59.6} & \textbf{3.79} & \textbf{195.3} \\
\bottomrule
\end{tabular}%
}
\end{table*}

    

\subsection{Main Results}
\textbf{Results on LLaDA Model.}
Table \ref{tab:main_results} demonstrates that the TAD framework successfully improve the accuracy-parallelism trade-off across two flexible configurations. Across all benchmarks, the Quality model (TAD-Q) achieves the highest average accuracy of 51.6\%, substantially outperforming the 46.2\% baseline average. On the complex MATH dataset where baselines typically degrade, TAD-Q improves accuracy to 42.7\% while decoding 4.49 tokens per forward (TPF). The Speed model (TAD-S) attains the highest average AUP score of 257.1. On GSM8K-CoT, TAD-S reaches 8.47 TPF with 78.8\% accuracy, achieving a better trade-off compared to d3LLM~\cite{qian2026d3llm} and dParallel~\cite{chen2025dparallel}. Both models also enhance code generation, with TAD-Q surpassing the baseline HumanEval accuracy while decoding approximately 6 tokens per forward. These results confirm the effectiveness of our trajectory collection mechanism and temporal-aware distillation framework.


\textbf{Results on Dream Model.}
Applying TAD to the Dream architecture (Table \ref{tab:TAD-dream-wide}) confirms its strong generalization. TAD-Q achieves the highest average accuracy of 61.3\%, peaking at 64.1\% on HumanEval to significantly outperform the original model. On MATH, TAD-Q maintains a strong 42.8\% accuracy while TAD-S achieves the peak AUP score of 161.3 at 5.30 TPF. Both configurations avoid the severe performance drops on GSM8K-CoT and MBPP typical of training-free methods. Although baselines like d3LLM occasionally exhibit marginally higher raw throughput, TAD-S delivers the highest average AUP score of 195.3. These results confirm that our TAD framework improves the accuracy-parallelism trade-off.

\begin{table}[h]
\centering
\caption{Ablation study on the distillation objectives. `Hard CE' refers to cross-entropy with reference tokens from trajectories, and `Soft KL' refers to KL divergence with teacher distributions.}
\label{tab:ablation_loss}
\begin{tabular}{c c | c c c | c c c}
\toprule
\multicolumn{2}{c|}{\textbf{Objectives}} & \multicolumn{3}{c|}{\textbf{MATH}} & \multicolumn{3}{c}{\textbf{HumanEval}} \\
Near & Distant & \textbf{TPF} $\uparrow$ & \textbf{Acc} $\uparrow$ & \textbf{AUP} $\uparrow$ & \textbf{TPF} $\uparrow$ & \textbf{Acc} $\uparrow$ & \textbf{AUP} $\uparrow$ \\
\midrule
Hard CE & Hard CE & 4.34 & 40.3 & \underline{159.1} & 5.37 & 32.9 & 72.1 \\
Soft KL & Soft KL & 2.08 & \underline{41.4} & 90.0 & 2.90 & 35.4 & 55.4 \\
Hard CE & None    & \textbf{5.12} & 34.8 & 142.9  & \textbf{6.47} & \underline{39.6} & \underline{110.7} \\
\midrule
\rowcolor{gray!10} 
\textbf{Hard CE} & \textbf{Soft KL} & \underline{4.49} & \textbf{42.7} & \textbf{190.4} & \underline{6.11} & \textbf{42.1} & \textbf{117.4} \\
\bottomrule
\end{tabular}
\end{table}

\begin{table}[h]
\centering
\caption{Ablation on data collection methods and distillation strategies. We report the average Accuracy, TPF, and AUP score across the four evaluated benchmarks.}

\label{tab:ablation_data}
\resizebox{\textwidth}{!}{%
\begin{tabular}{l c c c}
\toprule
\textbf{Data Source and Training Strategy} & \textbf{Acc (\%)} & \textbf{TPF} & \textbf{AUP} \\
\midrule
Standard SFT (Ground Truth + Random Mask) & 46.2 & 3.24 & 105.8 \\
Trajectory (with privileged information) + Random Mask & 47.9 & 4.26 & 152.8 \\
Trajectory (w/o privileged information) + TAD Distillation & 47.9 & 5.03 & 187.1 \\
Trajectory (with privileged information) + TAD Distillation (Ours) & \textbf{51.6} & \textbf{5.08} & \textbf{225.2} \\
\bottomrule
\end{tabular}
}
\end{table}

\subsection{Ablation Study}
We conduct ablation studies on the LLaDA-8B architecture to validate the design choices within the TAD framework.

\textbf{Effect of Decoupled Distillation Objectives.}
Table \ref{tab:ablation_loss} evaluates the individual objective components using a fixed window of $\delta=8$. Applying hard cross-entropy globally forces early confidence on distant tokens, lowering HumanEval accuracy to 32.9\%. Conversely, global soft KL divergence lacks deterministic targets to anchor the generation path, resulting in over-smoothed predictions and minimal acceleration (2.08 TPF on MATH). Restricting hard cross-entropy solely to the near subset accelerates generation but reduces MATH accuracy to 34.8\%, demonstrating the necessity of distant supervision. Combining near-term hard targets with distant soft supervision optimally resolves this problem, yielding the highest accuracy and AUP scores across both benchmarks.

\textbf{Impact of Privileged Information-Guided Trajectories.}
Table \ref{tab:ablation_data} evaluates four training paradigms to confirm the necessity of privileged information. Standard supervised fine-tuning with random masking ignores sequential dependencies and yields the lowest performance. Applying random masking to the final text of the privileged information-guided trajectory improves accuracy but still omits natural state transitions. Distilling from valid trajectories generated without privileged information increases TPF but limits accuracy. In contrast, distilling from trajectories generated with privileged ground-truth context achieves the highest average accuracy (51.6\%) and optimal AUP score (225.2), confirming that these paths provide essential high-quality targets.

\textbf{Sensitivity to the Spatial Partition Window ($\delta$).}
Table \ref{tab:delta_t_performance} presents average performance under varying window sizes. A conservative window ($\delta=4$) yields the highest accuracy (52.1\%) but restricts decoding speed. Expanding $\delta$ to 20 provides deterministic supervision to a larger sequence portion, accelerating generation to 5.76 TPF and achieving the peak AUP score (257.1) with only a minor accuracy decline. However, an extreme window ($\delta=256$) mimics global cross-entropy, compelling the model to predict distant tokens without sufficient context and severely degrading performance. These observations justify our dual-model strategy, assigning a moderate window for robust reasoning (Quality model) and a larger window for maximized throughput (Speed model).

\begin{table}[htbp]
\centering
\begin{minipage}[t]{0.48\textwidth}
\centering
\caption{Average performance across four benchmarks under varying partition window sizes ($\delta$).}
\label{tab:delta_t_performance}
\begin{tabular}{lccc}
\toprule
$\delta$ & \textbf{Acc (\%)}  & \textbf{TPF} & \textbf{AUP} \\
\midrule
$\delta=4$  & \textbf{52.1} & 4.33 & 190.3 \\
$\delta=8$  & 51.6 & 5.08 & 225.2 \\
$\delta=12$ & 49.9 & 4.64 & 196.4 \\
$\delta=16$ & 48.9 & 5.57 & 229.2 \\
$\delta=20$ & 49.9 & 5.76 & \textbf{257.1} \\
$\delta=24$ & 48.9 & \textbf{5.78} & 237.1 \\
$\delta=256~(\text{Global CE})$  & 46.3 & 5.10 & 208.4 \\
\bottomrule
\end{tabular}
\end{minipage}\hfill
\begin{minipage}[t]{0.48\textwidth}
\centering
\caption{Average performance across four benchmarks under varying KL weights ($\lambda$).}
\label{tab:lambda_performance}
\begin{tabular}{lccc}
\toprule
$\lambda$ & \textbf{Acc (\%)} & \textbf{TPF} & \textbf{AUP} \\
\midrule
$\lambda=0.0$ & 47.2 & \textbf{5.35} & 192.1 \\
$\lambda=0.5$ & 51.1 & 4.90 & 215.9 \\
$\lambda=1.0$ & \textbf{51.6} & 5.08 & \textbf{225.2} \\
$\lambda=1.5$ & 50.5 & 4.89 & 207.5 \\
$\lambda=2.0$ & 50.4 & 4.74 & 199.5 \\
\bottomrule
\end{tabular}
\end{minipage}
\end{table}

\textbf{Sensitivity to the weight of $\mathcal{L}_{distant}$ ($\lambda$).}
Table \ref{tab:lambda_performance} evaluates the balance between near cross-entropy and distant Kullback-Leibler divergence. Removing distant supervision ($\lambda=0$) maximizes speed (5.35 TPF) but causes the lowest average accuracy (47.2\%), confirming that omitting distant dependencies damages generation quality. Integrating the distant constraint at $\lambda=1.0$ restores reasoning capabilities, achieving the optimal average accuracy (51.6\%) and AUP score (225.2). Increasing the weight further ($\lambda \geq 1.5$) overemphasizes the soft objective, reducing the impact of near-term certainty forcing and subsequently decreasing.


\section{Related Work}

\subsection{Diffusion Large Language Models}
Recent research has extended diffusion modeling from continuous domains~\cite{croitoru2023diffusion} to discrete text generation~\cite{austin2021structured, sahoo2024simple, lou2024discrete, shi2024simplified}. Unlike traditional autoregressive models that rely on left-to-right sequential generation~\cite{achiam2023gpt, guo2025deepseek, yang2025qwen3}, Diffusion Large Language Models (dLLMs) feature bidirectional context attention and parallel decoding capabilities~\cite{li2025survey}. Recent models, including the LLaDA series~\cite{nie2025large, zhu2025llada, bie2025llada2, bie2026llada2}, Dream~\cite{ye2025dream}, and SDAR~\cite{cheng2025sdar}, achieve performance competitive with leading autoregressive models across various benchmarks. They also demonstrate advantages in reverse reasoning tasks that require global planning~\cite{nie2025large}. Beyond these developments, the research community has increasingly focused on enhancing reasoning capabilities~\cite{zhao2025d1,wang2025revolutionizing,tang2025wd1,ou2025principled,liu2026efficient}, building agent systems~\cite{zhen2026dllm,zhao2026dllm}, and accelerating inference~\cite{lin2026efficient} for dLLMs. In this paper, we focus on further accelerating dLLM inference by increasing the parallelism of these models.

\subsection{Inference Acceleration for dLLMs}
The inference speed of dLLMs is primarily hindered by the incompatibility of traditional KV caching with bidirectional attention and the severe quality degradation during highly parallel decoding~\cite{wu2025fast}. To alleviate the caching bottleneck, recent studies~\cite{liu2025dllm,ma2025dkv,jiang2025d,wu2025fast} exploit the temporal consistency of KV states across decoding iterations to develop approximate caching mechanisms, significantly reducing redundant computations. To enhance parallelism, current approaches are categorized into training-free~\cite{wu2025fast,hong2025wide,xu2025lopa,shen2025improving,wu2025free,wang2025creditdecoding} and training-based~\cite{chen2025dparallel,wang2025diffusion,kim2025cdlm,qian2026d3llm,zhang2026t3d,liang2026cd4lm,bao2025learning,hu2026lightningrl} methods. Training-free strategies accelerate inference by dynamically adapting decoding schedules, but their effectiveness is bounded by the capacity of the model. Alternatively, training-based methods finetune the model for parallel generation. While they achieve higher throughput, this acceleration typically comes at the expense of generation quality. Building upon the training-based paradigm, our work improves this trade-off through a privileged-information strategy to acquire high-quality trajectories and a temporal-aware distillation framework.

\section{Limitations}
\label{sec: limitations}
As a training-based method, TAD mainly has three limitations. First, the framework depends on high-quality ground-truth responses to generate trajectory, which restricts its applicability in unsupervised settings. Second, the token-by-token rollout during trajectory collection introduces overhead prior to distillation. Third, the partition window $\delta$
requires empirical tuning across architectures. We leave dynamic window sizing and data-efficient trajectory generation for future work.

\section{Conclusion}
\label{sec: conclusion}
We present TAD, a temporal-aware trajectory self-distillation framework that improves the accuracy-parallelism trade-off. The framework collects high-quality trajectories via a teacher conditioned on privileged information and partitions masked positions by their decoding steps, applying cross-entropy to near tokens for throughput and KL divergence loss to distant tokens for dependency preservation. Experiments on mathematical reasoning and code generation confirm that this design improves both accuracy and decoding speed, offering a practical path toward deploying efficient dLLMs.

\medskip

{
\small
\bibliographystyle{unsrt} 

\bibliography{custom} 
}


\appendix

\section{Proof of Theoretical Analysis}

\subsection{Equivalence of Joint KL Divergence and Expected Cross-Entropy}
\label{sec:proof}

We prove that minimizing the Kullback-Leibler divergence between the teacher's sequential joint distribution and the student's factorized distribution is equivalent to minimizing the expected sum of token-wise cross-entropies along the decoding trajectory. This result formally connects the theoretical objective in Eq.~\ref{eq:theoretical_obj} of the main text to the trainable objective in Eq.~\ref{eq:simplified_obj}.

\paragraph{Notation.} Throughout this proof, we adopt the same notation as Section 3.1. Let $\mathbf{x}_t$ denote a teacher-generated intermediate masked state, $q$ the prompt, $\mathcal{M} = \{m_1, \ldots, m_K\}$ the ordered set of masked positions revealed by the teacher's token-by-token rollout, and $\mathbf{x}_0^{\mathcal{M}} = (x_0^{m_1}, \ldots, x_0^{m_K})$ the corresponding clean tokens. The teacher's TBT joint distribution and the student's factorized distribution are
\begin{align}
p_{\theta_T}^{TBT}(\mathbf{x}_0^{\mathcal{M}} \mid \mathbf{x}_t, q) &= \prod_{k=1}^{K} p_{\theta_T}(x_0^{m_k} \mid \mathbf{x}_t, x_0^{m_1}, \ldots, x_0^{m_{k-1}}, q), \label{eq:teacher_joint} \\
p_{\theta_S}(\mathbf{x}_0^{\mathcal{M}} \mid \mathbf{x}_t, q) &= \prod_{k=1}^{K} p_{\theta_S}(x_0^{m_k} \mid \mathbf{x}_t, q), \label{eq:student_factorized}
\end{align}
For brevity, we abbreviate $\mathbf{x}_0^{m_k}$ as $x_k$ and $(x_0^{m_1}, \ldots, x_0^{m_{k-1}})$ as $x_{<k}$ in the derivation below.

\begin{theorem}
\label{thm:equivalence}
Under the factorization in Eq.~\ref{eq:student_factorized}, for any fixed intermediate state $\mathbf{x}_t$, minimizing the joint KL divergence
\begin{equation}
\min_{\theta_S}\; D_{\mathrm{KL}}\!\left( p_{\theta_T}^{TBT}(\mathbf{x}_0^{\mathcal{M}} \mid \mathbf{x}_t, q) \,\big\|\, p_{\theta_S}(\mathbf{x}_0^{\mathcal{M}} \mid \mathbf{x}_t, q) \right)
\label{eq:thm_kl}
\end{equation}
is equivalent to minimizing the expected sum of token-wise cross-entropies along the teacher's decoding path:
\begin{equation}
\min_{\theta_S}\; \sum_{k=1}^{K} \mathbb{E}_{x_{<k} \sim p_{\theta_T}^{TBT}}\!\left[ -\sum_{x_k} p_{\theta_T}(x_k \mid \mathbf{x}_t, x_{<k}, q)\, \log p_{\theta_S}(x_k \mid \mathbf{x}_t, q) \right].
\label{eq:thm_ce}
\end{equation}
\end{theorem}

\begin{proof}
By the definition of KL divergence,
\begin{equation}
D_{\mathrm{KL}}\!\left( p_{\theta_T}^{TBT} \,\|\, p_{\theta_S} \right) = \mathbb{E}_{\mathbf{x}_0^{\mathcal{M}} \sim p_{\theta_T}^{TBT}}\!\left[ \log p_{\theta_T}^{TBT}(\mathbf{x}_0^{\mathcal{M}} \mid \mathbf{x}_t, q) - \log p_{\theta_S}(\mathbf{x}_0^{\mathcal{M}} \mid \mathbf{x}_t, q) \right].
\label{eq:kl_def}
\end{equation}
Since the first term does not depend on $\theta_S$, it acts as a constant in the optimization. Minimizing Eq.~\ref{eq:thm_kl} is therefore equivalent to minimizing the cross-entropy
\begin{equation}
\min_{\theta_S}\; - \mathbb{E}_{\mathbf{x}_0^{\mathcal{M}} \sim p_{\theta_T}^{TBT}}\!\left[ \log p_{\theta_S}(\mathbf{x}_0^{\mathcal{M}} \mid \mathbf{x}_t, q) \right].
\label{eq:ce_only}
\end{equation}
Applying the factorization in Eq.~\ref{eq:student_factorized}, the joint log-probability decomposes into a sum of marginal log-probabilities:
\begin{equation}
\log p_{\theta_S}(\mathbf{x}_0^{\mathcal{M}} \mid \mathbf{x}_t, q) = \sum_{k=1}^{K} \log p_{\theta_S}(x_k \mid \mathbf{x}_t, q).
\label{eq:log_decomp}
\end{equation}
Substituting Eq.~\ref{eq:log_decomp} into Eq.~\ref{eq:ce_only} and exchanging the finite sum with the expectation,
\begin{equation}
\min_{\theta_S}\; \sum_{k=1}^{K} \left( - \mathbb{E}_{\mathbf{x}_0^{\mathcal{M}} \sim p_{\theta_T}^{TBT}}\!\left[ \log p_{\theta_S}(x_k \mid \mathbf{x}_t, q) \right] \right).
\label{eq:swap_sum}
\end{equation}
For each index $k$, the integrand $\log p_{\theta_S}(x_k \mid \mathbf{x}_t, q)$ depends only on $x_k$. By the law of total expectation, we marginalize the joint expectation over $x_{<k}$, $x_k$, and $x_{>k}$ in turn:
\begin{equation}
\mathbb{E}_{\mathbf{x}_0^{\mathcal{M}} \sim p_{\theta_T}^{TBT}}\!\left[ \log p_{\theta_S}(x_k \mid \mathbf{x}_t, q) \right] = \mathbb{E}_{x_{<k}}\!\left[ \mathbb{E}_{x_k \mid x_{<k}}\!\left[ \mathbb{E}_{x_{>k} \mid x_{\leq k}}\!\left[ \log p_{\theta_S}(x_k \mid \mathbf{x}_t, q) \right] \right] \right],
\label{eq:total_exp}
\end{equation}
where all conditional distributions on the right-hand side are induced by $p_{\theta_T}^{TBT}$. Because $\log p_{\theta_S}(x_k \mid \mathbf{x}_t, q)$ is constant with respect to $x_{>k}$, the innermost expectation reduces to $\log p_{\theta_S}(x_k \mid \mathbf{x}_t, q)$ itself. Expanding the remaining expectation over $x_k$,
\begin{equation}
\mathbb{E}_{\mathbf{x}_0^{\mathcal{M}} \sim p_{\theta_T}^{TBT}}\!\left[ \log p_{\theta_S}(x_k \mid \mathbf{x}_t, q) \right] = \mathbb{E}_{x_{<k} \sim p_{\theta_T}^{TBT}}\!\left[ \sum_{x_k} p_{\theta_T}(x_k \mid \mathbf{x}_t, x_{<k}, q)\, \log p_{\theta_S}(x_k \mid \mathbf{x}_t, q) \right].
\label{eq:inner_expand}
\end{equation}
Substituting Eq.~\ref{eq:inner_expand} back into Eq.~\ref{eq:swap_sum} yields
\begin{equation}
\min_{\theta_S}\; \sum_{k=1}^{K} \mathbb{E}_{x_{<k} \sim p_{\theta_T}^{TBT}}\!\left[ -\sum_{x_k} p_{\theta_T}(x_k \mid \mathbf{x}_t, x_{<k}, q)\, \log p_{\theta_S}(x_k \mid \mathbf{x}_t, q) \right],
\label{eq:final_form}
\end{equation}
which matches Eq.~\ref{eq:thm_ce} and completes the proof.
\end{proof}

\paragraph{Remark.} Theorem~\ref{thm:equivalence} formalizes the key insight underlying the TAD framework: the student's per-position marginal predictions must match the teacher's conditional probabilities along the sampled trajectory in order to internalize the sequential dependencies that a fully factorized model would otherwise miss. This conditional matching is exactly what motivates the hard cross-entropy loss on near tokens (Section~\ref{sec:self_distillation}), where the teacher's conditional distribution becomes sharply peaked and is well approximated by a one-hot label.

\begin{algorithm}[htbp]
\caption{Temporal-Aware Trajectory Self-Distillation (TAD)}
\label{alg:tad}
\textbf{Input:} Trajectory dataset $\mathcal{D} = \{ (q, a, \tau_{priv}) \}$, temporal partition window $\delta$, temperature $\tau$, loss weights $\lambda$. \\
\textbf{Parameters:} Trainable student model $\theta_S$, frozen teacher model $\theta_T$. \\
\textbf{Output:} Optimized student model $\theta_S^*$.
\begin{algorithmic}[1]
\WHILE{not converged}
    \STATE Sample a batch of tuples $(q, a, \tau_{priv})$ from dataset $\mathcal{D}$
    \STATE Sample a decoding step $s$ and obtain the intermediate masked state $x_s$ from $\tau_{priv}$
    \STATE Obtain the look-ahead target state $x_{target} \leftarrow x_{s+\delta}$ from $\tau_{priv}$
    
    \STATE \textit{\# 1. Privilege-Aware Input Construction}
    \STATE Construct student input: $I_S \leftarrow \text{Concat}(q, x_s)$
    \STATE Construct teacher input: $I_T \leftarrow \text{Concat}(q, a, x_s)$
    
    \STATE \textit{\# 2. Spatial Partitioning of Masked Tokens}
    \STATE Initialize near subset $\mathcal{M}_{near} \leftarrow \emptyset$, distant subset $\mathcal{M}_{distant} \leftarrow \emptyset$
    \STATE Initialize hard label sequence $Y \leftarrow \emptyset$
    \FOR{each position $i$ in $x_s$}
        \IF{$x_s^i == \text{[MASK]}$ \AND $x_{target}^i \neq \text{[MASK]}$}
            \STATE $\mathcal{M}_{near} \leftarrow \mathcal{M}_{near} \cup \{i\}$ \COMMENT{Near tokens to be decoded within $\delta$ steps}
            \STATE $Y^i \leftarrow x_{target}^i$ 
        \ELSIF{$x_s^i == \text{[MASK]}$ \AND $x_{target}^i == \text{[MASK]}$}
            \STATE $\mathcal{M}_{distant} \leftarrow \mathcal{M}_{distant} \cup \{i\}$ \COMMENT{Distant tokens remaining masked}
        \ENDIF
    \ENDFOR
    
    \STATE \textit{\# 3. Model Forward Pass}
    \STATE $L_S \leftarrow \theta_S(I_S)$ \COMMENT{Student marginal logits}
    \STATE $L_T \leftarrow \theta_T(I_T)$ \COMMENT{Teacher conditional logits (no\_grad)}
    
    \STATE \textit{\# 4. Decoupled Objective Computation}
    \STATE $\mathcal{L}_{near} \leftarrow \text{CrossEntropy}(L_S[\mathcal{M}_{near}], Y[\mathcal{M}_{near}])$ 
    
    \STATE $P_S \leftarrow \text{LogSoftmax}(L_S[\mathcal{M}_{distant}] / \tau)$
    \STATE $P_T \leftarrow \text{Softmax}(L_T[\mathcal{M}_{distant}] / \tau)$
    \STATE $\mathcal{L}_{distant} \leftarrow \tau^2 \cdot \text{KLDiv}(P_S, P_T)$ 
    
    \STATE $\mathcal{L}_{TAD} \leftarrow  \mathcal{L}_{near} + \lambda \mathcal{L}_{distant}$
    
    \STATE \textit{\# 5. Optimization}
    \STATE Update student parameters: $\theta_S \leftarrow \theta_S - \eta \nabla_{\theta_S} \mathcal{L}_{TAD}$
\ENDWHILE
\end{algorithmic}
\end{algorithm}

\section{Algorithm}
In this section, we describe the training algorithm of TAD. Algorithm~\ref{alg:tad} provides the pseudocode of the full training procedure.

\section{More Implementation Details}
\label{sec:implementation_details}
\subsection{Training Details}

We apply TAD to two dLLMs: LLaDA-8B-Instruct and Dream-7B-Instruct.
Both models are fine-tuned using LoRA with DeepSpeed ZeRO-Stage 2 on 8$\times$H200 GPUs,
using bfloat16 mixed precision throughout. The total training time is 20 hours for each model. The detailed training hyperparameters are summarized in Table~\ref{tab:training_hyperparams}.

\paragraph{Sequence Length.}
We do not impose a fixed maximum sequence length during training.
Instead, each training sample is constructed by concatenating the tokenized prompt
with the trajectory token sequence from the corresponding diffusion step,
and sequences within each mini-batch are dynamically right-padded to the longest sample in that batch. Padding positions are excluded from all loss computations via label masking.

\begin{table}[h]
\centering
\caption{Training hyperparameters for TAD.}
\label{tab:training_hyperparams}
\begin{tabular}{lcc}
\toprule
\textbf{Hyperparameter} & \textbf{TAD-LLaDA} & \textbf{TAD-Dream} \\
\midrule
\multicolumn{3}{l}{\textit{Training}} \\
Optimizer & AdamW & AdamW \\
Learning Rate & $2 \times 10^{-5}$ & $2 \times 10^{-5}$ \\
LR Scheduler & Constant & Constant \\
Warmup Steps & 0 & 0 \\
Weight Decay & 0.01 & 0.01 \\
Max Gradient Norm & 1.0 & 1.0 \\
Training Epochs & 1 & 1 \\
Per-Device Batch Size & 4 & 4 \\
Gradient Accumulation Steps & 8 & 8 \\
Global Batch Size & 256 & 256 \\
Precision & bfloat16 & bfloat16 \\
DeepSpeed Stage & ZeRO-2 & ZeRO-2 \\
Max Sequence Length & Dynamic & Dynamic \\
\midrule
\multicolumn{3}{l}{\textit{LoRA}} \\
Rank ($r$) & 128 & 128 \\
Alpha ($\alpha$) & 128 & 128 \\
Dropout & 0.05 & 0.05 \\
Bias & None & None \\
Target Modules & \multicolumn{2}{c}{\texttt{q,k,v,o,gate,up,down}\_proj} \\
\midrule
\multicolumn{3}{l}{\textit{Distillation}} \\
$\lambda$ & 1.0 & 1.0 \\
Temperature ($\tau$) & 1.0 & 1.0 \\
\bottomrule
\end{tabular}
\end{table}

\subsection{Evaluation Details}
\label{appendix:eval_details}

\subsubsection{Evaluation Metrics}

We adopt the AUP (Accuracy Under Parallelism)~\cite{qian2026d3llm} score as our primary evaluation metric. AUP is defined as a weighted area under the accuracy--parallelism curve, where parallelism is measured by Tokens Per Forward pass (TPF). Given a set of parallelism-accuracy pairs $S = \{(\rho_i, y_i)\}_{i=1}^{m}$ sorted by increasing parallelism $\rho_1 < \rho_2 < \cdots < \rho_m$, AUP is computed as:
\begin{equation}
\mathrm{AUP} \triangleq \rho_1 y_1 + \sum_{i=2}^{m} (\rho_i - \rho_{i-1}) \left( \frac{y_i W(y_i) + y_{i-1} W(y_{i-1})}{2} \right),
\end{equation}
where the weighting function $W(y) = \min(e^{-\alpha(1 - y / y_{\max})}, 1)$ penalizes accuracy degradation relative to the best accuracy $y_{\max}$ achieved on that task. We use $\alpha = 3$ as the default penalty factor. A minimum accuracy threshold $y_{\min} = y_1 - 5$ is applied to exclude regimes of severe accuracy degradation. AUP rewards methods that increase parallelism without sacrificing accuracy, while suppressing contributions from low-accuracy regimes. Importantly, AUP is hardware-independent since it relies on TPF rather than tokens per second (TPS), providing a fair comparison of algorithmic parallelism across different hardware setups.

\subsubsection{Evaluation Configurations}

We evaluate on four downstream benchmarks covering math reasoning and code generation. All evaluations are conducted using the \texttt{lm-evaluation-harness} framework. Due to differences in model architecture and instruction-following capabilities, TAD-LLaDA and TAD-Dream use slightly different task variants, as summarized in Table~\ref{tab:eval-config}.
During inference, we set the maximum generation length to 256 tokens for all tasks. We use greedy decoding with temperature set to 0.0. The block size is fixed at 32 tokens for all experiments. For our TAD framework with multi-block generation, the block-add threshold is set to 0.1, and the decoded token threshold is 0.95.
\begin{table}[h]
\centering
\caption{Evaluation configurations for TAD-LLaDA and TAD-Dream.}
\label{tab:eval-config}
\resizebox{\textwidth}{!}{%
\begin{tabular}{llcccccc}   
\toprule
\textbf{Model} & \textbf{Benchmark} & \textbf{Few-shot} & \textbf{Gen Length} & \textbf{Block Size} & \textbf{Entropy threshold} \\
\midrule
\multirow{4}{*}{TAD-LLaDA}
 & GSM8K              & 0-shot & 256 & 32 & 0.5   \\
 & MATH (Minerva)     & 4-shot & 256 & 32 & 0.5   \\
 & HumanEval          & 0-shot & 256 & 32 & 0.5   \\
 & MBPP               & 3-shot & 256 & 32 & 0.5   \\
\midrule
\multirow{4}{*}{TAD-Dream}
 & GSM8K & 0-shot & 256 & 32 & 0.5   \\
 & MATH (Minerva)         & 4-shot & 256 & 32 & 0.5   \\
 & HumanEval-Instruct     & 0-shot & 256 & 32 & 0.5   \\
 & MBPP-Instruct          & 0-shot & 256 & 32 & 0.5   \\
\bottomrule
\end{tabular}%
}
\end{table}

\section{More Experiment Results}
\subsection{Impact of Privileged Information on Trajectory Collection}
\label{additional_experiment_results}

In this section, we investigate the quality of the intermediate trajectories collected during the data construction phase. Specifically, we evaluate the generative accuracy of the teacher models under two distinct rollout conditions: (1) standard self-generation, where the model is conditioned solely on the input question without access to the ground truth (w/o GT), and (2) generation with privileged information, where the model receives both the input question and the ground-truth response. 

As presented in Table \ref{tab:model_performance}, the inclusion of ground-truth information yields a substantial and consistent improvement in trajectory accuracy across all evaluated benchmarks for both LLaDA-8B-Instruct and Dream-7B-Instruct architectures. For instance, on the GSM8K Training Set, the privileged information guidance improves LLaDA's accuracy from 68.88\% to 89.16\%, and Dream's from 75.17\% to 94.79\%. This performance confirms that leveraging ground truth as privileged information is beneficial to sample the high-quality intermediate states required for effective temporal-aware self-distillation.

\begin{table}[h]
    \centering
    \caption{Performance comparison of dLLMs with and without Ground Truth (GT).}
    \label{tab:model_performance}
    \begin{tabular}{l l c c c}
        \toprule
        \textbf{Model} & \textbf{Condition} & \textbf{GSM8K}  & \textbf{KodCode} \\
        \midrule
        \multirow{2}{*}{LLaDA-8B-Instruct} 
            & w/o GT   & 68.88  & 8.22 \\
            & with GT  & \textbf{89.16}  & \textbf{77.28} \\
        \midrule
        \multirow{2}{*}{Dream-7B-Instruct} 
            & w/o GT   & 75.17  & 14.05 \\
            & with GT  & \textbf{94.79}  & \textbf{75.08} \\
        \bottomrule
    \end{tabular}
\end{table}

\subsection{Throughput Analysis}
To further validate the practical deployment efficiency of the TAD framework, we conduct a wall-clock throughput analysis on the GSM8K-CoT benchmark using NVIDIA H200 GPUs. We measure the generation speed in tokens per second (TPS) and compare it against the base models and strong acceleration baselines. The results for LLaDA and Dream architectures are presented in Table~\ref{tab:tad_combined}.

\textbf{Results on LLaDA Architecture}
As shown in Table~\ref{tab:tad_combined}, TAD demonstrates a remarkable improvement over the baseline and existing acceleration methods on the LLaDA architecture. The Quality mode (TAD-LLaDA-Q) achieves 339.4 TPS, a 10.9-fold speedup over LLaDa, while simultaneously improving the accuracy from 72.6\% to 79.9\%. Furthermore, the Speed mode (TAD-LLaDA-S) maximizes hardware utilization, reaching a peak throughput of 451.8 TPS (a 14.5-fold acceleration) while maintaining a robust accuracy of 78.8\%. 

\textbf{Results on Dream Architecture}
The Dream base model inherently possesses a highly optimized initial state, achieving 83.9\% accuracy on GSM8K-CoT. As illustrated in Table \ref{tab:tad_combined}, all acceleration methods incur a minor accuracy penalty on this architecture. TAD-Dream-Q achieves 205.7 TPS (a 5.1-fold speedup) while preserving a highly competitive accuracy of 81.4\%. TAD-Dream-S pushes the throughput to 288.4 TPS (a 7.2-fold speedup) with a marginal drop to 81.0\% accuracy. We observe that while the absolute peak TPS of TAD-Dream-S is slightly lower than that of d3LLM~\cite{qian2026d3llm} (295.0 TPS), TAD matches its accuracy in the Quality mode. 

\begin{table}[htbp]
\centering
\caption{Throughput comparison of \textit{TAD-LLaDA} and \textit{TAD-Dream} on GSM8K-CoT using H200 GPUs. We report tokens per second (TPS) and accuracy (\%). Speedup ratios relative to the respective base models (LLaDA and Dream) are shown in parentheses.}
\label{tab:tad_combined}

\begin{subtable}[t]{0.48\linewidth}
    \centering
    \footnotesize 
    \setlength{\tabcolsep}{4pt} 
    \begin{tabular}{lcc}
        \toprule
        \textbf{Method} & \textbf{H200 TPS $\uparrow$} & \textbf{Acc (\%) $\uparrow$} \\
        \midrule
        LLaDA                & 31.2 ($1.0\times$)  & 72.6 \\
        Fast-dLLM-LLaDA      & 137.3 ($4.4\times$) & 74.7 \\
        D2F-LLaDA            & 121.1 ($3.9\times$) & 73.2 \\
        dParallel-LLaDA      & 214.3 ($6.9\times$) & 72.6 \\
        d3LLM-LLaDA          & 328.2 ($10.5\times$)& 73.1 \\
        TAD-LLaDA-Q          & 339.4 ($10.9\times$)& \textbf{79.9} \\
        TAD-LLaDA-S          & \textbf{451.8 ($14.5\times$)} & 78.8 \\
        \bottomrule
    \end{tabular}
    \label{tab:tad_llada}
\end{subtable}
\hfill 
\begin{subtable}[t]{0.48\linewidth}
    \centering
    \footnotesize
    \setlength{\tabcolsep}{4pt}
    \begin{tabular}{lcc}
        \toprule
        \textbf{Method} & \textbf{H200 TPS $\uparrow$} & \textbf{Acc (\%) $\uparrow$} \\
        \midrule
        Dream                & 40.1 ($1.0\times$)   & 83.9 \\
        Fast-dLLM-Dream      & 117.8 ($2.9\times$)  & 79.0 \\
        Fast-dLLM-v2         & 169.3 ($4.2\times$)  & 77.5 \\
        dParallel-Dream      & 189.4 ($4.7\times$)  & 82.1 \\
        d3LLM-Dream          & \textbf{295.0 ($7.4\times$)} & 81.4 \\
        TAD-Dream-Q          & 205.7 ($5.1\times$)  & 81.4 \\
        TAD-Dream-S          & 288.4 ($7.2\times$)  & 81.0 \\
        \bottomrule
    \end{tabular}
    \label{tab:tad_dream}
\end{subtable}

\end{table}

\end{document}